\documentclass[twocolumn]{article}
\usepackage[twocolumn,textwidth=16.8cm,columnsep=.9cm]{geometry}
\usepackage{sectsty}
\sectionfont{\fontsize{12}{15}\selectfont}
\subsectionfont{\fontsize{10}{13}\selectfont}

\usepackage{microtype}
\usepackage{graphicx}
\usepackage{subfigure}
\usepackage{booktabs} 

\usepackage{amssymb,amsmath,amsthm}
\usepackage{empheq}
\usepackage{algorithm,algorithmic}
\usepackage{microtype}
\usepackage{graphicx,wrapfig}
\usepackage{booktabs}
\usepackage{hyperref}

\newcommand{\hor}{H}

\def\A{{\mathcal A}}

\newcommand{\ignore}[1]{}

\def\bold0{\mathbf{0}}

\def\epsilon{\varepsilon}

\newcommand{\defeq}{\triangleq}

\newtheorem{theorem}{Theorem}[section]

\newtheorem{definition}[theorem]{Definition}

\newtheorem{assumption}[theorem]{Assumption}

\newcommand{\newreptheorem}[2]{%
\newenvironment{rep#1}[1]{%
 \def\rep@title{#2 \ref{##1}}%
 \begin{rep@theorem}}%
 {\end{rep@theorem}}}

\newreptheorem{theorem}{Theorem}
\newreptheorem{lemma}{Lemma}
\newreptheorem{proposition}{Proposition}
\newreptheorem{claim}{Claim}
\newreptheorem{corollary}{Corollary}
\newreptheorem{mainlemma}{Main Lemma}

\newcommand{\namedref}[2]{\mbox{\hyperref[#2]{#1~\ref*{#2}}}}

\newcommand{\figurerefb}[2]{\mbox{\hyperref[#1]{Figure~\ref*{#1}#2}}}

\newcommand{\equationref}[1]{\mbox{\hyperref[#1]{(\ref*{#1})}}}

\title{\textbf{Boosting for Control of Dynamical Systems}}
\author{
  Naman Agarwal$^{1}$ \and Nataly Brukhim$^{1\,2}$ \and Elad Hazan$^{1\,2}$ \and Zhou Lu$^{2}$\\
  \\
  $^1$ Google AI Princeton \\
  $^2$ Department of Computer Science, Princeton University \\
  \texttt{namanagarwal@google.com}, \texttt{\{nbrukhim,ehazan,zhoul\}@princeton.edu}\\
}
\date{}
\begin{document}

\maketitle


\begin{abstract}
We study the question of how to aggregate controllers for dynamical systems in order to improve their performance. 
To this end, we propose a framework of boosting for online control. 
Our main result is an efficient boosting algorithm that combines weak controllers into a provably more accurate one. 
Empirical evaluation on a host of control settings supports our theoretical findings. 
\end{abstract}
\section{Introduction}

In many learning scenarios it is significantly easier to come up with a mildly accurate rule of thumb than state of the art performance. This motivation led to the development of ensemble methods and boosting \cite{schapire2012boosting}, a theoretically sound methodology to combine rules of thumb (often referred to as weak learners) into a substantially more accurate learner. 

The application of boosting has transformed machine learning across a variety of applications, including supervised learning: classification \cite{Freund:1997:DGO:261540.261549}, regression \cite{mason2000boosting}, online learning \cite{beygelzimer2015optimal,beygelzimer2015online}, agnostic learning \cite{kanade2009potential}, recommendation systems \cite{Freund:2003:EBA:945365.964285} and many more.

While the same motivation for boosting exists for  dynamical systems, i.e. it is often easy to come up with a reasonable predictor or a controller for a dynamical system, the theory and practice of boosting faces significant challenges in these settings due to the existence of a state. Formally, a dynamical system is specified by a rule $x_{t+1} = f(x_t, u_t) + w_t$. The problem of optimal control in dynamical systems requires the design of a sequence of controls $\{u_t\}$ so as to minimize a certain objective (for example making the states follow a certain trajectory). As can be seen readily, the decisions made by a controller affects the future trajectory of the system, and hence it is often not a-priori clear how to obtain a meaningful guarantee when switching between or aggregating different controllers. 

In this paper we propose a framework for formalizing boosting in the context of optimal control of dynamical systems. The first crucial insight comes from the newly emerging literature on non-stochastic control, which allows a non-counterfactual description of the dynamics. 
Then leveraging techniques from online learning with memory \cite{anava2015online} and online gradient boosting \cite{beygelzimer2015optimal}, we provide a boosting algorithm for controlling systems with state and prove appropriate theoretical guarantees on its performance.     

\paragraph{The notion of boosting in online learning.}
Notice that unlike supervised learning, in the case of dynamical systems it is not immediately clear what metric should be used to measure the improvement provided by boosting algorithms. The theory of online gradient boosting \cite{beygelzimer2015optimal} suggests to boost online learning by improving the {\it expressivity} of the predictors. Given a set of weak online learners, online boosting guarantees prediction (as measured by mistake bounds or regret) that is as good as a predictor in a larger class. 

Furthermore, a robust way to study the optimal control problem is to study it in the online non-stochastic setting \cite{agarwal2019online}, where both the objective to be minimized and the perturbations to the system get revealed in an online fashion\footnote{We formally define the notion of non-stochastic control in Section \ref{subsec:NSC}.}.

Motivated by this, we take an online boosting approach to optimal control. Our boosting algorithm when given access to (weak) controllers from a certain class, provides a boosted controller than can provably perform (in terms of regret) as well as a controller from the larger class of a committee (or convex combination) of the weak controllers. Furthermore, we provide an alternate boosting algorithm, which is more efficient in terms of the number of weak controllers required, and allows for the utilization of weak controllers designed for handling only quadratic losses into strong controllers that can compete against the more general class of smooth and strongly-convex losses. 

We provide the relevant background and setup in Section \ref{sec:setting}. We formally describe our algorithms and provide formal guarantees on their performance in Section \ref{sec:boosting_ds}. We conclude with experimental evaluation of our methods on a variety of control and learning tasks in dynamical systems. 

\paragraph{Dynamics with bounded memory.}
For our boosting techniques to induce bounded regret we require that the dynamical systems considered have negligible dependence on history beyond a certain amount of time steps in the past. We quantify this property as $H$-bounded memory and formally define it in Definition \ref{def:stable}. This assumption is analogous to a standard assumption in Reinforcement Learning, called \textit{mixability} of the underlying Markov Decision Process. The so called \textit{episodic} setting in RL is also often used to circumvent long-term dependencies. 

In control theory, the bounded memory assumption manifests in two forms. The stronger notion called \textit{stability}, posits that the system remains bounded over any sequence of actions. This is often considered to be a very strong assumption. A weaker assumption commonly used is \textit{stablizability} or \textit{controllability}, which posits the existence of a controller or a policy which generates stable actions. Upon action with such a controller, the system de-facto exhibits a bounded memory. In this paper, we assume that all policy classes we work with induce a bounded memory on the dynamical system (or analogously are \textit{fast-mixing} or \textit{stabilizing}).

\paragraph{Boosting vs. overparametrization of deep controllers.} As opposed to supervised learning, in control there are many situations in which there is limited availability of training data, if at all. While deep neural networks have proven extremely successful for large data regimes in supervised learning, we experimentally find that for control, boosting small network controllers results in superior performance to training of a large network controller.

\section{Background and Setup}\label{sec:setting}

Our treatment of boosting applies only to certain dynamical systems and control methods. The main requirement we place is a memory bound on the effect of past actions. This section formally describes the dynamics setting and underlying assumptions for our algorithms to apply. 

\subsection{Non-stochastic Control} \label{subsec:NSC}

A general framework of robust optimal control has emerged recently which rests on analyzing optimal control in an online non-stochastic setting \cite{agarwal2019online,hazan2019nonstochastic,simchowitz2020improper}. In this framework, at each round $t \in [T]$ the controller observes the state of the system $x_t \in \mathbb{R}^k$, and outputs an action $u_t \in \mathcal{U} \subset \mathbb{R}^d$, where $\mathcal{U}$ is a convex bounded set. The adversary then reveals a convex cost function and the loss $c_t(x_t, u_t)$ is incurred by the controller. The system then transitions to a new state $x_{t+1}$ according to the following law with $f$ representing the dynamics of the system, 
\begin{equation}\label{eq:ds}
x_{t+1}=f(x_t,u_t)+w_t,
\end{equation}
where $w_t \in \mathbb{R}^k$ is an adversarially chosen perturbation to the dynamics that the system suffers at each time step. The costs and the perturbations are not known to the controller in advance and are assumed to be revealed to the controller after it has committed to the action $u_t$. The task of the controller is to minimize regret defined in the following way

\[\mathrm{Regret} = \sum_{t} c_t(x_t, u_t) - \min_{\pi \in \Pi} \sum_t c_t(x_t(w_{1:t}, \pi), u_t^{\pi})\]

Here $\Pi$ represents a class of policies that we wish to compare to. Furthermore $x_t(w_{1:t}, \pi)$ is the state that the system would have reached when executing $\pi$ on the perturbed dynamics with the \textit{same} perturbations $\{w_1 \ldots w_t\}$. Observe that this is a counterfactual notion of regret, since the actions performed by the comparator affect the future cost suffered by the comparator.

Note that the assumption of observable perturbation is without loss of generality when the underlying system $f$ is known to the controller and the state is fully observable. We do not make these assumptions in the paper but rather work with the setting of the complete observation of $w$ as in \cite{agarwal2019online}. Furthermore, we make no distributional assumptions on $w_t$ and only assume $\left\lVert w_t \right\rVert_2 \le W$, for some $W >0$.

\paragraph{Perturbation Based Policies}
The reference class of policies $\Pi$ we consider in this paper is comprised of policies $\pi$ that map a sequence of perturbations $w_1 \ldots w_{t-1}$ to an action $u_t$. Note that this class of policies is only more general than the standard notion of policies which map the current $x_t$ to an action $u_t$. In particular,  it captures linear policies for linear dynamical systems (Section \ref{sec:boosting_control}). A crucial property of such policies is that the decisions depend directly on the underlying dynamics and do not depend directly on the control feedback, but only implicitly via the perturbations. 

Another important limitation we place on policies and dynamics in this paper is memory boundedness, as we now define. 
\begin{definition}[$(H,\epsilon)$-Bounded Memory] \label{def:stable}
Given a dynamical system as given in Equation \ref{eq:ds}, for a sequence of actions $u_1, \ldots, u_T$ and any time $t$, let $\hat{x}_t$ be the state reached by the system if we artificially set $x_{t-H+1} = 0$ and simulate the system with the actions $u_{t-H+1}, \ldots ,u_t$. 

The sequence of actions $\{u_1 \ldots u_T\}$ is considered to be of {\bf $(H, \epsilon)$-bounded memory} if for all $t$,
$$|c_t(\hat{x}_t,u_t)- c_t(x_t,u_t) | \le \epsilon.$$
\end{definition}

\begin{assumption}[Bounded Memory of Convex Combinations]\label{ass:convex-stab}
For a given dynamical system the class of $(H,\epsilon)$-memory bounded sequences is closed under convex combination. 
\end{assumption}

Note that the notion of bounded memory is a slightly stronger notion than that of \textit{controllability} in the sense that it is applicable to changing policies as well. This notion is also referred to as sequential strong stability in the work of \cite{cohen2018online}. The key point is that the effect of the distant past (beyond $H$ most recent actions and disturbances) on the current state is negligible. Multiple previous works exhibit policy classes which produce bounded memory actions \cite{cohen2018online, agarwal2019online, pmlr-v97-cohen19b}. Concretely, in Section \ref{sec:boosting_control} we describe the GPC controller from \cite{agarwal2019online} which is shown to be memory bounded as well as satisfy Assumption \ref{ass:convex-stab}. 

\paragraph{Reduction to Bounded Memory Functions}
 For a sequence of actions, we now define a proxy function which only penalizes the last $H$ actions \footnote{Each $\ell_t$ naturally depends on the sequence of $\{w_t:w_{t-H}\}$ chosen by the adversary. We suppress this dependence for notational convenience.}:
\begin{equation}\label{eq:proxy_cost}
    \ell_t(\mathbf{0},u_{t-H+1}, \ldots, u_t) := c_t(\hat{x}_t, u_t)
\end{equation}
$(H,\epsilon)$-bounded memory of actions now ensures that minimizing regret over the proxy costs $\ell_t$, which have finite memory, is sufficient to minimize overall regret. Having reduced the control of dynamical systems to minimizing regret over functions with memory, we are now ready to discuss the technique of Online Boosting, which we will apply on the proxy cost function.

\subsection{Online Boosting}\label{subsec:online_boosting}
The presence of state in non-stochastic control makes online boosting more challenging. We first give a brief background on online boosting for the regression setting \cite{beygelzimer2015online}, and in Section \ref{sec:boosting_ds} discuss a reduction that enables the use of a similar technique for non-stochastic control.

Informally, online boosting refers to a meta-learning algorithm which is given black-box oracle access to an online (\textit{weak}) learning algorithm $\mathcal{A}$ for a function class $\Pi$ and \textit{linear} losses, with regret $R$, and is given a bound $N$ on the total number of calls made in each
iteration to copies of $\mathcal{A}$. The algorithm then obtains an online learning algorithm $\mathcal{A}'$ for a richer function class
$\Pi' = \text{conv}(\Pi)$ (i.e. the convex hull of $\Pi$), and any \textit{convex} losses, with a (possibly larger) regret $R'$.

The online booster maintains $N$ instances of the weak learning algorithm, denoted $\mathcal{A}_1,...,\mathcal{A}_N$. In each round $t \in [T]$, an adversary selects an example $x_t$ from a compact feature space $\mathcal{X}$, and a loss function $\ell_t: \mathcal{X} \rightarrow \mathbb{R}^k$, and presents $x_t$ to the
 learner. The policy regret $R(T)$ of each weak learner $\mathcal{A}_i$ is assumed to be bounded as
$$ 
\sum_{t=1}^T \ell_t(\mathcal{A}_i(x_t)) - \min_{\pi \in \Pi} \sum_{t=1}^T  \ell_t(\pi(x_t)) \le R(T),
$$
with the regret $R(T)$ being a non-decreasing sub-linear function of $T$. Note the slight abuse of notation here; $\mathcal{A}_i(\cdot)$ is not a function but rather the output of the online learning algorithm $\mathcal{A}_i$ computed on the given example using its internal state. In each round $t$. the online booster takes some convex combination of all the predictions made by the learners, and outputs the boosted prediction. To update $\mathcal{A}_1,...,\mathcal{A}_N$  at every iteration $t \in [T]$, the booster passes a carefully chosen loss function to each of the weak learners. Specifically, each learner $\mathcal{A}_i$ is fed with a \textit{residual} loss function $\ell_t^i(y) = \nabla(y_t^{i-1}) \cdot y$, where $y_t^{i-1}$ is a convex combination of previous weak learner predictions, $\mathcal{A}_1(x_t),...,\mathcal{A}_{i-1}(x_t)$.

The work of \cite{beygelzimer2015online} 
proves that this technique results in a regret bound of
$$
\sum_{t=1}^T \ell_t(y_t) - \min_{\pi \in \text{conv}(\Pi)} \sum_{t=1}^T  \ell_t(\pi(x_t)) \le R(T) + O\bigg( \frac{T}{N} \bigg).
$$
where $y_1,...,y_T$ are the predictions outputted by the booster. Note that although the regret of the boosting algorithm is larger by $O(T/N)$ than the regret of the weak learners, it is now achieved against the best predictor available in a richer class. This is especially meaningful when the class of predictors $\Pi$ is e.g., neural networks, a highly non-convex class. Thus, by boosting such predictors, the resulting algorithm is guaranteed to have low regret with respect to the convex hull of  $\Pi$. Our method is based on the online boosting technique, as detailed next. 
\section{Algorithms and Main Results}\label{sec:boosting_ds}

This section describes our algorithms for boosting in dynamical systems. The main idea of our methods is to leverage the memory boundedness and reduce online control of dynamical systems to online learning with finite memory \cite{anava2015online}. We achieve this by constructing a proxy cost function which only takes into account the $H$ most recent rounds of the system (see Equation \ref{eq:proxy_cost}). 
We then extend the online boosting methodology (discussed in Subsection \ref{subsec:online_boosting}) to apply to these losses with memory. Bounded memory ensures that minimizing regret over our constructed proxy costs is sufficient to minimize overall regret.

We propose two algorithms (\ref{alg1}, \ref{alg2}) for boosting online control, given access to \textit{weak} controllers (see definitions below) which obtain low regret against a policy class $\Pi$ and class of losses $\mathcal{L}$. 
For the first algorithm we assume $\mathcal{L}$ to be the class of \textit{linear} losses as detailed in Subsection \ref{subsec:linear_loss_boosting}. For the second algorithm we assume $\mathcal{L}$ to be the class of \textit{quadratic} losses as detailed in Subsection \ref{subsec:linear_loss_boosting}.

\begin{table}[t]
\begin{tabular}{|c|c|c|c|}
\hline
Algorithm &  
Class &
Loss &
Regret\\
\hline
DBoost 1 &  $\text{conv}(\Pi)$ & linear & $R+ T/N$\\
\hline
DBoost 2 &  $\Pi$ & quadratic &  $R+T(1-\frac{\alpha}{\beta})^N$ \\
\hline
\end{tabular}
\caption{\textbf{Main results summary.} Boosting uses $N$ weak controllers which have low regret $R =o(T)$, against a reference class of predictors $\Pi$. The \textit{DynaBoost1} algorithm allows to compete with
the best committee (convex combination) of weak controllers conv$(\Pi)$. \textit{DynaBoost2}, which is more efficient (requires smaller $N$), suited for losses that are $\alpha$-strongly convex and $\beta$-smooth.
\textit{DynaBoost1} and \textit{DynaBoost2} assume weak controller guarantees hold w.r.t. linear and quadratic losses, respectively.
}
\label{table:result_summary}
\end{table}

 Although the second method requires stronger assumptions, its advantage is that it is more efficient in terms of the number of copies $N$ of weak controllers required to achieve low regret. 


\subsection{\textit{DynaBoost1}: Boosting Online Control}\label{subsec:linear_loss_boosting}
Consider the non-stochastic control setting described in Subsection \ref{subsec:NSC}, for a dynamical system as defined in Equation \ref{eq:ds}.
DynaBoost1 is presented as Algorithm \ref{alg1} and assumes an oracle access to a weak controller, which is defined as follows:
\begin{definition}\label{def:weak_controller}
Let $\mathcal{A}_i$ be an online learning algorithm for a dynamical system as defined in Equation \ref{eq:ds} and a reference policy class $\Pi$. The learner $\mathcal{A}_i$ is a \textbf{weak controller} with respect to a class of loss functions $\mathcal{L}$ if 
\begin{enumerate}
    \item The sequence of actions produced by $\mathcal{A}_i$ is of  $(H,\epsilon)$-bounded memory (see Definition \ref{def:stable}).
    \item When run with losses $\ell_t^i$ chosen from the class of loss functions $\mathcal{L}$, it produces a sequence of actions $u_1,...,u_T$ s.t.,
\[
\sum_{t=1}^T \ell_t^i(u_1,...,u_t) - \min_{\pi \in \Pi}\sum_{t=1}^T \ell_t^i(u_1^{\pi},...,u_t^{\pi}) \le R(T).
\]
where action $u_t^{\pi}$ is obtained by applying $\pi \in \Pi$ the best policy in hindsight, and the regret $R(T)$ is a non-decreasing sub-linear function of the horizon $T$. 
\end{enumerate}
\end{definition}
 We can now construct the proxy \textit{linear} cost functions $\ell_t^i(u_{t-H+1}, \ldots, u_t)$ which only consider the $H$ most recent rounds (see line 11 of Algorithm \ref{alg1}), thus obtaining the following regret guarantee,
\begin{align}\label{eq:weak_control_regret}
    \sum_{t=1}^T \ell_t^i  (u_{t-H+1}, ..., u_t) - & \min_{\pi \in \Pi} \sum_{t=1}^T \ell_t^i(u_{t-H+1}^{\pi}, ... ,u_t^{\pi}) \notag \\ 
        & \qquad \le \quad R(T) + 2T\epsilon.
\end{align}
Before stating our main theorem, we need the following definition.  We say that a loss function $\ell$ is $\beta$-smooth if for all $u_1, \ldots, u_H$ and $\tilde{u}_1, \ldots ,\tilde{u}_H$, it holds that,
\begin{align}\label{ass:smooth_loss}
    & \ell(u_1, \ldots, u_H)  - \ell(\tilde{u}_1, \ldots ,\tilde{u}_H) \qquad \\
    &\le \sum_{j=1}^H \nabla_j \ell(\tilde{u}_1, \ldots ,\tilde{u}_H)^{\top}(u_j - \tilde{u}_j)+\frac{\beta}{2}\sum_j \lVert u_j-\tilde{u}_j \lVert_2^2 \notag 
\end{align}

\begin{algorithm}[t]
\caption{DynaBoost 1}
\label{alg1}
\begin{algorithmic}[1]
\STATE Maintain $N$ weak learners $\mathcal{A}_1$,...,$\mathcal{A}_N$. \\
\STATE Set step length $\eta_i=\frac{2}{i+1}$ for $i \in [N]$.
\FOR{$t = 1, \ldots, T$}
\STATE Receive the state $x_t$.
\STATE Define $u_t^0=\mathbf{0}$.
\FOR{$i=1$ to $N$}
\STATE Define $u_t^i=(1-\eta_i)u_t^{i-1}+\eta_i \mathcal{A}_i(x_t)$.
\ENDFOR
\STATE Output action $u_t=u_t^N$.
\STATE Receive loss $\ell_t$, suffer $\ell_t(u_1, \ldots , u_t)$.
\STATE Define linear loss function:
$$\ell_t^i(\mathbf{u}_1...\mathbf{u}_H) \defeq \sum_{j=1}^\hor \nabla_j^\top  \mathbf{u}_j$$
where, $\nabla_j := \nabla_{t-H+j} \ell_t(\mathbf{0},u_{t-H+1}^{i-1},...,u_t^{i-1})$.
\STATE Pass loss function $\ell_t^i(\cdot)$ to weak controller $\mathcal{A}_i$.
\ENDFOR
\end{algorithmic}
\end{algorithm}

Under these assumptions we can now give our main theorem, providing a regret bound for Algorithm \ref{alg1}. 
\begin{theorem} \label{thm:main1}
Let $\mathcal{L}'$ be the class of $\beta$-smooth loss functions. Assume oracle access to $N$ copies of a weak controller $\mathcal{A}$ (see Definition \ref{def:weak_controller}) satisfying Equation \ref{eq:weak_control_regret}. Let $D_{\mathcal{U}}$ be the diameter of the action set  $\mathcal{U}$. Then, there exists a boosting algorithm (Algorithm \ref{alg1}) which produces a sequence of actions $u_t$ for which the following regret bound holds with respect to the reference class $\mathrm{conv}(\Pi)$,
\begin{align*}
     \sum_{t=1}^T \ell_t (u_1,...,u_t) - &\min_{\pi \in \mathrm{conv}(\Pi)} \sum_{t=1}^T \ell_t(u_1^{\pi},...,u_t^{\pi}) \\ \\
    &\le \frac{2\beta D_{\mathcal{U}}^2 H T}{N}+ R(T) + 4T\epsilon. 
\end{align*}
\end{theorem}
The proof is given in Section \ref{sec:analysis}.

\subsection{\textit{DynaBoost2}: Fast-Boosting Online Control} \label{subsec:quad_loss_boosting}

We now present our results for the case when the loss functions we compete with are strongly convex. In this case we prove that the excess regret of boosting goes down exponentially in the number of weak learners. The weak learners required for this result are stronger in the sense that they are able to have low regret against quadratic functions as opposed to linear functions in the previous part. Due to this, the boosted algorithm does not compete with an expanded class of predictors but rather just with the original class of predictors $\Pi$.

In addition to assumptions in the previous subsection we will need the following additional assumptions.
We say a loss function $\ell$ is $\alpha$-strongly convex when for all $u_1, \ldots, u_H$ and $\tilde{u}_1, \ldots ,\tilde{u}_H$,
\begin{align}\label{ass:alpha_strong}
    & \ell(u_1, \ldots, u_H)  - \ell(\tilde{u}_1, \ldots ,\tilde{u}_H) \qquad \\
    &\ge \sum_{j=1}^H \nabla_j \ell(\tilde{u}_1, \ldots ,\tilde{u}_H)^{\top}(u_j - \tilde{u}_j)+\frac{\alpha}{2}\sum_j \lVert u_j-\tilde{u}_j \lVert_2^2 \notag
\end{align}
Furthermore we say $\ell$ is $G$-bounded if for all $u_1, \ldots, u_H \in \cal{U}$ we have that $|\ell(u_1, \ldots, u_H)| \leq G$.

Under these assumptions we can now give our main theorem, providing a regret bound for Algorithm \ref{alg2}. \\

\begin{theorem} \label{thm:main2}
Let $\mathcal{L}'$ be the class of $\alpha$ strongly convex and $G$-bounded loss functions. Assume oracle access to $N$ copies of a weak controller $\mathcal{A}$ (see Definition \ref{def:weak_controller}) satisfying Equation \ref{eq:weak_control_regret} with respect to the class $\mathcal{L}$ of $\alpha$-strongly convex quadratic functions. Then, there exists a boosting algorithm (Algorithm \ref{alg2}) which produces a sequence of actions $u_t$ for which the following regret bound holds with respect to the reference class $\Pi$,
\begin{align*}
     \sum_{t=1}^T \ell_t (u_1,...,u_t) - &\min_{\pi \in \Pi} \sum_{t=1}^T \ell_t(u_1^f,...,u_t^f) \\ \\
    &\le (1-\frac{\alpha}{\beta})^N 2GT + R(T) + 4T\epsilon. 
\end{align*}
\end{theorem}

\begin{algorithm}[t]
\caption{DynaBoost 2}
\label{alg2}
\begin{algorithmic}[1]
\STATE Maintain $N$ weak learners $\mathcal{A}_1$,...,$\mathcal{A}_N$. \\
\STATE Set step length $\eta_i=\frac{\alpha}{\beta}$ for $i \in [N]$.
\FOR{$t = 1, \ldots, T$}
\STATE Receive the state $x_t$.
\STATE Define $u_t^0=\mathbf{0}$.
\FOR{$i=1$ to $N$}
\STATE Define $u_t^i=(1-\eta_i)u_t^{i-1}+\eta_i \mathcal{A}_i(x_t)$.
\ENDFOR
\STATE Output action $u_t=u_t^N$.
\STATE Receive loss $\ell_t$, suffer $\ell_t(u_1, \ldots , u_t)$.
\STATE Define quadratic loss function 
\[
\ell_t^i(\mathbf{u}_1...\mathbf{u}_H) \defeq \sum_{j=1}^H \frac{\eta_i\beta}{2}\|\mathbf{u}_j -u_{t-H+j}^{i-1}\|^2 \quad +\qquad \qquad \qquad 
\]
\[
\sum_{j=1}^{\hor} \left( \nabla_j^\top  (\mathbf{u}_j-u_{t-H+j}^{i-1})\right).
\]
\STATE where, $\nabla_j := \nabla_{t-H+j} \ell_t(\mathbf{0},u_{t-H+1}^{i-1},...,u_t^{i-1})$.
\STATE Pass loss function $\ell_t^i(\cdot)$ to weak controller $\mathcal{A}_i$.
\ENDFOR
\vskip -0.3in
\end{algorithmic}
\end{algorithm}

We provide the proof of Theorem \ref{thm:main1} next. The proof of Theorem \ref{thm:main2} which follows a similar argument is deferred to the Appendix. 

\section{Proof of Theorem \ref{thm:main1}} \label{sec:analysis}

\begin{proof}
First, note that for any $i=1,2 \ldots N$, since $\ell_t^i \in \mathcal{L}$, the loss function encountered by the weak controller (defined in Line 11 of \ref{alg1}), is a linear function, we have that:
\begin{align*}
\min_{\pi \in \mathrm{conv}(\Pi)} \sum_{t=1}^T & \ell_t^i(u_{t-H+1}^{\pi},...,u_t^{\pi}) \\
& = \min_{\pi \in\Pi} \sum_{t=1}^T \ell_t^i(u_{t-H+1}^{\pi},...,u_t^{\pi})
\end{align*}
Now let $\pi$ be any function in $\mathrm{conv}(\Pi)$. Observe that by the equality above and the regret bound of the weak controller (Equation \ref{eq:weak_control_regret}), we get,
\begin{align}\label{eq:pf_weak_control_regret}
    \sum_{t=1}^T \bigg(\ell_t^i (u_{t-H+1}, ..., u_t) - & \ell_t^i(u_{t-H+1}^{\pi}, ... ,u_t^{\pi})\bigg) \notag \\
        & \quad \le  R(T) +  2T\epsilon.
\end{align}
Denote $j^-= t-j+1$ for brevity. Define for any $i \in [N]$, $t \in [T]$, and any $\ell_t \in \mathcal{L}'$ loss function encountered by the booster,
\[
\Delta_{t,i} \defeq \ell_t (\mathbf{0},u^i_{H^-}, ..., u^i_t) - \ell_t(\mathbf{0},u_{H^-}^{\pi}, ... ,u_t^{\pi}).
\]
Consider the following calculations for $\Delta_{t,i}$:

\begin{align*}
    &\Delta_{t,i} = \ell_t\bigg(\mathbf{0},u_{H^-}^{i{-}1}{+} \eta_i(\mathcal{A}_i(x_{H^-}) - u_{H^-}^{i-1}), \ldots, \\
    & \qquad  \quad u_t^{i-1}{+} \eta_i(\mathcal{A}_i(x_t) - u_t^{i-1})\bigg) - \ell_t(\mathbf{0},u_{H^-}^{\pi},...,u_t^{\pi}) \\
    &  \qquad (\text{by substituting $u^i_t$ as in line 7 of Algorithm \ref{alg1}})\\
    & \qquad \leq \ell_t(\mathbf{0},u^{i-1}_{H^-}, ..., u^{i-1}_t) - \ell_t(\mathbf{0},u_{H^-}^{\pi}, ... ,u_t^{\pi})  + \\ 
    &\qquad \quad \sum_{j=1}^{\hor} \bigg( \eta_i \nabla_j^{\top} (\mathcal{A}_i(x_{t-H+j}) - u_{t-H+j}^{i-1}){+}\\
    & \qquad \qquad \quad \frac{\eta_i^2 \beta}{2} \|\mathcal{A}_i(x_{t-H+j}) - u_{t-H+j}^{i-1}\|^2 \bigg)  \\
    &  \quad (\text{by convexity and $\beta$-smoothness of $\ell_t$,} \\
    & \qquad \text{ and definition of $\nabla_j$ (line 11, Algorithm \ref{alg1})})
\end{align*}
Denote $\Delta_i = \sum_t \Delta_{t, i}$. Then, by summing over $t \in [T]$, and applying the weak-controller regret bound (Equation \ref{eq:pf_weak_control_regret}), we have,
\begin{align*}
    &\Delta_{i} \leq \sum_{t=1}^{T} \Bigg( \left(\ell_t(\mathbf{0},u^{i-1}_{H^-}, ..., u^{i-1}_t) - \ell_t(\mathbf{0},u_{H^-}^{\pi}, ... ,u_t^{\pi})\right) + \\ 
    & \qquad \quad \eta_i \sum_{j=1}^{\hor} \nabla_j^{\top}(u_{t-H+j}^{\pi}{-}u_{t-H+j}^{i-1}) \Bigg){+}\\
    & \qquad \quad \eta_i (R(T){+}2T\epsilon){+}\frac{\eta_i^2 \beta D_{\mathcal{U}}^2 \hor T}{2}  \\
    &\quad \leq (1 - \eta_i)\Delta_{i-1} + \eta_i (R(T){+}2T\epsilon){+}\frac{\eta_i^2 \beta D_{\mathcal{U}}^2 \hor T}{2}
\end{align*}

where we used the bound $\|\mathcal{A}_i(x_{t-H+j}) - u_{t-H+j}^{i-1}\|^2 \le 2D_{\mathcal{U}}$.
For $i=1$, since $\eta_1 = 1$, the above bound implies that $\Delta_1 \leq \frac{\beta D_{\mathcal{U}}^2 \hor T}{2} + (R(T) +  2T\epsilon) $. Starting from this base case, by induction on $i \geq 1$ it follows that $\Delta_i \leq \frac{2\beta D_{\mathcal{U}}^2 \hor T}{i} + (R(T) +  2T\epsilon) $. Applying the above bound for $i=N$ yields the desired result for truncated memory losses. Lastly, using Assumption \ref{ass:convex-stab} completes the proof. 
\end{proof}

\section{Case Studies} \label{sec:boosting_control}

For the sake of clarity, we precisely spell out the application of our boosting algorithms with two choices of weak learning methods to illustrate the general technique of applying our boosting algorithm. 

\subsection{Boosting Deep Controllers} \label{subsec:case_study_rnn}

Consider a controller based on a Recursive Neural Network(RNN) for the non-stochastic control problem with dynamics \eqref{eq:ds}. As motivated earlier we explicitly enforce the $H$-memory bounded property via the choice of the sequence length of the RNN. Formally, the weak learners in this setting are deep neural networks $\mathrm{RNN}_{\theta}$ that map a sequence of $H$ past perturbations $w_{t:t-H} = w_t,...,w_{t-H}$ to control:
$$ u_{t+1} = \mathrm{RNN}_{\theta}(w_{t-H:t}) . $$
Here by $\theta$ we denote the internal weights of the network. 

When used inside Algorithm \ref{alg1}, each weak leaner $\mathrm{RNN}_{\theta} = \A_i$ is an instance of neural net that is initialized arbitrarily. Iteratively, the network $\mathrm{RNN}_{\theta}$ receives $w_t$ and predicts $u_{t+1}^i$ using a sequential feed forward computation over $w_{t-H:t}$.  It then receives the residual loss function $\ell_t^i(\cdot)$. 
It then applies the back-propagation algorithm to update its internal weights.

\subsection{Boosting for Linear Dynamical Systems} \label{subsec:case_study_gpc}

A linear dynamical system is governed by the dynamics equation
\begin{equation}\label{eqn:lds}
    x_{t+1} = Ax_t + Bu_t + w_t,  
\end{equation}
The system is assumed to be known and strongly stable(See Definition 3.3 in \cite{agarwal2019online}). We use the controller presented in \cite{agarwal2019online}(referred to as Gradient Perturbation Controller (GPC)) as the weak learners. The GPC controller parameterizes the control actions $u_t$ via the following equation:
\begin{equation}\label{eqn:controller}
    u_t=-K x_t+\sum_{i=1}^H M^i w_{t-i}
\end{equation}
where $K$ is a fixed pre-computed matrix (depending only on $A,B$) and $M=(M^1,\ldots M^H)$ are parameters governing the controller over which the controller learns. As shown in \cite{agarwal2019online}, $K$ can be selected such that the strong stability property of the system implies that the actions $u_t$ are $\left(O\left(\log(T/\epsilon)\right),\epsilon \right)$-bounded memory (see Theorem 5.3 in \cite{agarwal2019online}). Furthermore it can be easily checked that the actions also satisfy Assumption \ref{ass:convex-stab}. 

Having setup the weak controller thus we feed it inside Algorithm \ref{alg1}. Similar to the setting with the deep networks, iteratively, the controller recieves $w_t$ and predicts $u_{t+1}^i$ using the GPC prediction. Furthermore, it then receives the residual loss function $\ell_t^i(\cdot)$ and the internal parameters are updated according to the GPC update.

\section{Experiments} \label{sec:experiments}

\begin{figure*}[t]
  \includegraphics[width=\textwidth]{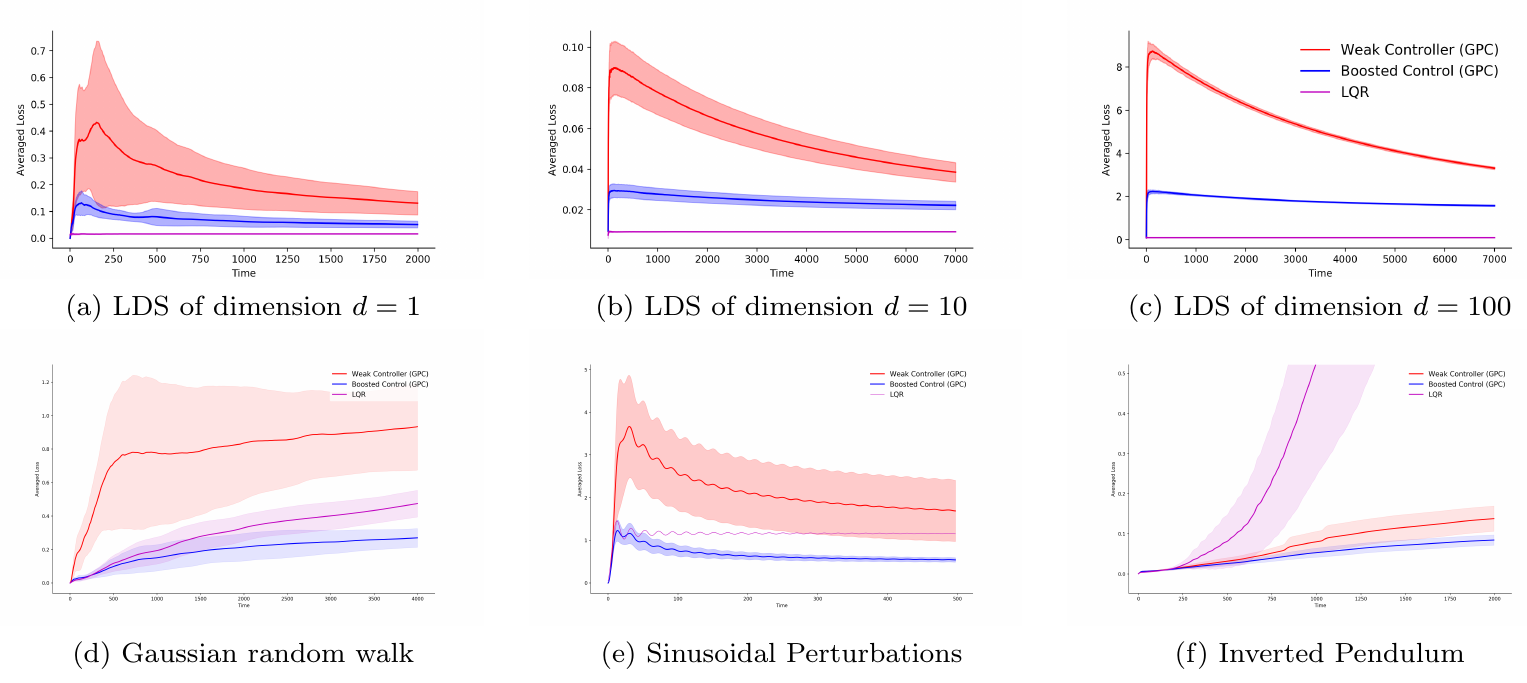}
    \caption{Online Boosting Control with GPC weak-controllers}%
    \label{fig:plots1}%
\end{figure*}
\begin{figure*}[t] \centering
  \includegraphics[width=13cm]{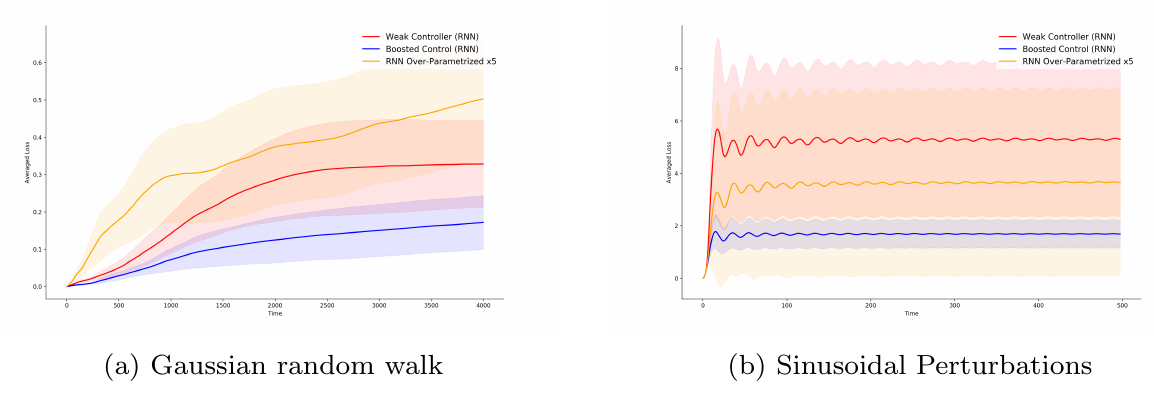}
    \caption{Online Boosting Control with RNN weak-controllers}%
    \label{fig:plots2}%
\end{figure*}
We have tested our framework of online boosting given in Algorithm \ref{alg1} in various control settings, as detailed below. 

The first weak-controller we have tested is the Gradient Perturbation Controller (GPC) discussed above (see Subsection \ref{subsec:case_study_gpc}), presented in Figure \ref{fig:plots1}. In addition, we also give results for a RNN-based controller (see Subsection \ref{subsec:case_study_rnn}), presented in Figure \ref{fig:plots2}. The weak-controller baselines, and the weak controllers fed to the boosting method, are the exact same controllers, with identical configuration per setting. Note that in all settings, weak-controllers performance (plotted in red) can be improved by applying boosting (plotted in blue). 

We begin with experiments on a Linear Dynamical System (as in Equation \ref{eqn:lds}) where the matrices $A,B$ are generated randomly. We then present experiments for a non-linear dynamics as well (Inverted Pendulum setting). The cost function used in all settings is $c(x,u) = \left\lVert x\right\rVert_2^2 + \left\lVert u\right\rVert_2^2$. The GPC weak-controller is designed as in Equation \ref{eqn:controller}, following \cite{agarwal2019online}, with the pre-fixed matrix $K$ set to $0$. The RNN weak-controller, using an LSTM architecture, with $5$ hidden units. In all figures, we plot the averaged results for a fixed system, which differs per setting, over $20$ experiment runs with different stochasticity. Confidence intervals of $.95$ are plotted in each setting as well. 

\paragraph{Sanity check experiments.}
To demonstrate the effectiveness of the system in terms of both (i) its ability to reach close to a known optimal controller, and (ii) its performance in different dimensions, we present the results of this setting, shown in the first row of Figure \ref{fig:plots1}. 
For the system used in each dimension $d \in \{1,10,100\}$ (with $d=k$ in all settings), each noise term $w_t$ is normally i.i.d. distributed with zero mean, and $0.1^2$ variance. We set the memory length to $H=5$, and use $N=5$ weak-learners in all the experiments. The Linear Quadratic Regulator (LQR) is known to be optimal in this setting and therefore this experiment only serves as a sanity check. 

\paragraph{Correlated disturbances experiments.}
We now consider more challenging LDS settings in which the disturbances $w_t$ are correlated across time.
In the "Gaussian random walk" setting, each noise term is distributed normally, with the previous noise term as its mean (specifically, $w_{t+1} \sim \mathcal{N}(w_t, 0.3^2)$), and is clipped to the range $[-1, 1]$. 
In the "Sinusoidal Perturbations" setting, the sine function is applied to the time index, such that, $w_t = \sin(t)/ 2\pi$. 

Note that in these settings the LQR method is no longer optimal due to perturbations being correlated across time. 
The RNN-based controllers perform better than GPC-based controllers in the "Gaussian random walk" setting, whereas in the "Sinusoidal Perturbations" setting, GPC outperforms RNNs. However, in both cases, Boosting improves upon its corresponding weak-controller. 

\paragraph{Boosting vs. Over-Parameterization} In Figure \ref{fig:plots2}, the "Over-parametrized RNN" baseline refers to a baseline controller of the same architecture and hyper-parameters as the RNN-weak controller, but with a larger hidden layer. We demonstrate that by using a larger network with overall same number of parameters as the boosted RNN controller, boosting achieves superior performance. Notice that enlarging the size of the network might result in a controller that is outperformed even by the smaller RNN controller, as in Figure \ref{fig:plots2}(a). 
Overall, this experiment implies that the strength of our method does not stem from using more parameters, but rather from the way in which the weak-controllers are maintained by the boosting framework.

\paragraph{Inverted Pendulum experiment.}
The inverted pendulum, 
a highly nonlinear unstable system, 
is a commonly used benchmark for control methods. The objective of the control system is to balance the inverted pendulum by applying torque that will stabilize it in a vertically upright position. 
Here we follow the dynamics that was implemented in \cite{openAI}. The LQR baseline solution is obtained from the linear approximation of the system dynamics, whereas our baseline and boosted controllers are not restricted to that approximation. We add correlated disturbances obtained from a Gaussian random walk, as above, such that $w_t \sim \mathcal{N}(w_{t-1}, \text{5e-3})$, where the noise values are then clipped to the range $[-0.5, 0.5]$.

\section{Conclusions}
We have described a framework for boosting of algorithms that have state information, and two efficient algorithms that provably enhance weak learnability in different ways. These can be applied to a host of control problems in dynamical systems. Preliminary experiments  in simulated control look promising, of boosting for both linear and deep controllers.



\bibliography{bib.bib}
\bibliographystyle{plain}

\appendix
\section{Appendix}

\subsection{Proof of Theorem \ref{thm:main2}}

\begin{proof}[Proof of Theorem \ref{thm:main2}]
	Since $\A^i$ satisfies inequality \ref{eq:weak_control_regret}, we have that
	\begin{align}\label{eq:pf_strong_control_regret}
        \sum_{t=1}^T \bigg(\ell_t^i (u_{t-H+1}, ..., u_t) - & \ell_t^i(u_{t-H+1}^f, ... ,u_t^f)\bigg) \notag \\
            & \quad \le  R(T) +  2T\epsilon.
    \end{align}

Denote $j^-= t-j+1$ for brevity. Define for any $i \in [N]$, $t \in [T]$, and any $\ell_t \in \mathcal{L}'$ loss function encountered by the booster,
\[
\Delta_{t,i} \defeq \ell_t (\mathbf{0},u^i_{H^-}, ..., u^i_t) - \ell_t(\mathbf{0},u_{H^-}^f, ... ,u_t^f).
\]

Denote $\Delta_i = \sum_t \Delta_{t, i}$. Notice that by $\alpha$-strongly convexity \ref{ass:alpha_strong} of $\ell_t$, as long as we choose $\eta_i\le \frac{\alpha}{\beta}$, we have

\begin{align*}
    &\ell_t^i(u_{t-H+1}^f, ... ,u_t^f)=\sum_{j=1}^H \frac{\eta_i\beta}{2}\|{u}_{t-H+j}^f -u_{t-H+j}^{i-1}\|^2 +\\
    & \sum_{j=1}^{\hor} \left( \nabla_j^\top  ({u}_{t-H+j}^f-u_{t-H+j}^{i-1})\right) \\
    &  \le \ell_t(\mathbf{0},u_{t-H+1}^f, ... ,u_t^f)-\ell_t(\mathbf{0},u_{t-H+1}^{i-1}, ..., u_t^{i-1}).\\
\end{align*}

Thus by summing them up we get
	\begin{equation} \label{eq:shalom}
		  \sum_{t=1}^T \ell_t^i ( u_{t-H+1}^{\pi}, ... ,u_t^{\pi} )    \leq - \Delta_{i-1}
	\end{equation}

Consider the following calculations for $\Delta_{t,i}$:

\begin{align*}
    \Delta_{t,i} &= \ell_t\bigg(\mathbf{0},u_{H^-}^{i{-}1}{+} \eta_i(\mathcal{A}^i(x_{H^-}) - u_{H^-}^{i-1}), \ldots, \\
    & \quad u_t^{i-1}{+} \eta_i(\mathcal{A}^i(x_t) - u_t^{i-1})\bigg) - \ell_t(\mathbf{0},u_{H^-}^{\pi},...,u_t^{\pi}) \\
    &  (\text{by substituting $u^i_t$ as in line 7 of Algorithm \ref{alg2}})\\
    &\leq \ell_t(\mathbf{0},u^{i-1}_{H^-}, ..., u^{i-1}_t) - \ell_t(\mathbf{0},u_{H^-}^{\pi}, ... ,u_t^{\pi})  + \\ 
    & \quad \sum_{j=1}^{\hor} \bigg( \eta_i \nabla_j^{\top} (\mathcal{A}^i(x_{t-H+j}) - u_{t-H+j}^{i-1}) \\
    &\qquad \qquad  \qquad + \frac{\eta_i^2 \beta}{2} \|\mathcal{A}^i(x_{t-H+j}) - u_{t-H+j}^{i-1}\|^2 \bigg)  \\
    &  (\text{by convexity and $\beta$-smoothness of $\ell_t$})
\end{align*}

By summing $\Delta_{t,i}$ over $t \in [T]$, we have that

\begin{align*}
    \Delta_{i} &\leq \sum_{t=1}^{T} \bigg( \left(\ell_t(\mathbf{0},u^{i-1}_{H^-}, ..., u^{i-1}_t) - \ell_t(\mathbf{0},u_{H^-}^{\pi}, ... ,u_t^{\pi})\right) + \\ 
    & \quad \eta_i \sum_{j=1}^{\hor} \bigg(\nabla_j^{\top} (\mathcal{A}^i(x_{t-H+j}) - u_{t-H+j}^{i-1}) \\
    &\qquad \qquad  \qquad + \frac{\eta_i \beta}{2} \|\mathcal{A}^i(x_{t-H+j}) - u_{t-H+j}^{i-1}\|^2 \bigg) \\
    &= \sum_{t=1}^{T} \bigg( \left(\ell_t(\mathbf{0},u^{i-1}_{H^-}, ..., u^{i-1}_t) - \ell_t(\mathbf{0},u_{H^-}^{\pi}, ... ,u_t^{\pi})\right) + \\
    & \eta_i \ell_t^i(\mathcal{A}^i(x_{H^-}), ... ,\mathcal{A}^i(x_{1^-})) \bigg) \\
    &\leq \sum_{t=1}^{T} \bigg( \left(\ell_t(u^{i-1}_{H^-}, ..., u^{i-1}_t) - \ell_t(u_{H^-}^{\pi}, ... ,u_t^{\pi})\right) + \\
    & \eta_i \ell_t^i((u_{t-H+1}^{\pi}, ... ,u_t^{\pi}) \bigg) +\eta_i (R(T)+2T\epsilon)\\
    &  (\text{by the weak-controller regret bound  \ref{eq:pf_strong_control_regret}}) \\
    &\leq (1 - \eta_i)\Delta_{i-1} + \eta_i (R(T){+}2T\epsilon) \quad (\text{by inequality \ref{eq:shalom}})
\end{align*}
	
Choosing $\eta_i=\frac{\alpha}{\beta}$, then by noticing $\Delta_i$ is always upper bounded by a convex combination of $\Delta_0$ and $(R(T)+2T\epsilon)$ , we have
\begin{align*}
    \Delta_i& \leq \left(1-\frac{\alpha}{\beta}\right)^i\Delta_0+\left(1-\left(1-\frac{\alpha}{\beta}\right)^i\right) (R(T)+2T\epsilon) \\
    &\leq \left(1-\frac{\alpha}{\beta}\right)^i 2GT+ R(T)+2T\epsilon
\end{align*}
plugging $i=N$ in finishes our proof.
\end{proof}

\end{document}